\documentclass{article}
\usepackage[utf8]{inputenc}
\usepackage{caption,graphicx,amsmath,amsthm,amssymb,xcolor,enumitem,subfigure,float,verbatim,hyperref,pgfplots,bibentry,mathtools,soul,isomath, cancel}
\floatstyle{plaintop}
\restylefloat{table}
\usepackage{mathrsfs}
\usepackage{indentfirst}
\usepackage{algorithm}
\usepackage{algorithmic}
\usepackage{booktabs}
\usepackage{nccmath}
\usepackage{pifont}
\usepackage{graphicx}
\usepackage{subcaption}
\usepackage{booktabs}
\usepackage{graphicx}
\usepackage{multirow}
\usepackage{multicol, MnSymbol}
\usepackage[dvipsnames]{xcolor}
\usetikzlibrary{shapes.geometric, arrows}
\usetikzlibrary{backgrounds}
\usetikzlibrary{patterns, shapes.misc, positioning}
\usepackage{pgfplots}
\usepackage{tikz}
\usepackage{xcolor}
\usepackage{float}
\usepackage{placeins}
\usetikzlibrary{arrows.meta, positioning, shapes.geometric}
\pgfplotsset{width=10cm,compat=1.9,tick scale binop=\times}
\usepackage{indentfirst}
\usepackage[a4paper,bindingoffset=0.2in,%
left=0.8in,right=1in,top=1.2in,bottom=1.2in,%
footskip=.25in]{geometry}
\usepackage{relsize}
\usepackage{appendix}
\usepackage[english]{babel}
\newtheorem{theorem}{Theorem}

\begin{document}
\title
{\bf{Designing a Robust, Bounded, and Smooth Loss Function for
Improved Supervised Learning}}
\author {\small Soumi Mahato
\footnote{Soumi\_p220127ma@nitc.ac.in}   and Lineesh M.C
\footnote{lineesh@nitc.ac.in} \\ \small Department of Mathematics, National Institute of Technology, Calicut,\\
\small Kerala, India-673601}
\date{}
\maketitle
\begin{abstract} 

The loss function is crucial to machine learning, especially in supervised learning frameworks. It is a fundamental component that controls the behavior and general efficacy of learning algorithms. However, despite their widespread use, traditional loss functions have significant drawbacks when dealing with high-dimensional and outlier-sensitive datasets, which frequently results in reduced performance and slower convergence during training. In this work, we develop a robust, bounded, and smooth (RoBoS-NN) loss function to resolve the aforementioned hindrances. The generalization ability of the loss function has also been theoretically  analyzed to rigorously justify its robustness. Moreover, we implement RoboS-NN loss in the framework of a neural network (NN) to forecast time series and present a new robust algorithm named $\mathcal{L}_{\text{RoBoS}}$-NN. To assess the potential of $\mathcal{L}_{\text{RoBoS}}$-NN, we conduct experiments on multiple real-world datasets. In addition, we infuse outliers into data sets to evaluate the performance of $\mathcal{L}_{\text{RoBoS}}$-NN in more challenging scenarios. Numerical results show that $\mathcal{L}_{\text{RoBoS}}$-NN outperforms the other benchmark models in terms of accuracy measures.

\vspace{0.5 cm}
 \noindent   \textbf{ Mathematis Subject Classifications: 68T07, 68T09. } 
\end{abstract}

\section{Introduction and Motivation}
In data analysis, classification and regression are fundamental tasks within the supervised machine learning paradigm, where models are trained on labeled data to predict outputs for unseen instances. 
 A central concept is the loss function, which quantifies the discrepancy between predicted and true outputs, thereby shaping the learning process. Neural networks (NNs) constitute a potent class of supervised machine learning algorithms capable of modeling intricate nonlinear relationships in data. Grounded in statistical learning theory and function approximation, NN training is performed by minimizing an empirical risk objective via gradient-based optimization. Their architectural flexibility and depth endow them with high representational capacity, enabling competitive generalization performance across diverse applications. In this paper, we carried out a detailed investigation of the relationship between loss functions and supervised learning within the neural network framework, highlighting how the choice of loss functions influences convergence behavior, robustness and predictive performance.\\ This study is focused on the regression task using a multilayer perceptron (MLP). Let the training dataset be defined as $\{x_k, y_k\}_{k=1}^n$, where $x_k \in \mathbb{R}^m$ denotes the input feature vector and $y_k \in \mathbb{R}$ represents the corresponding continuous target value. An MLP is a feedforward neural network composed of an input layer, one or more hidden layers, and an output layer, where each layer performs an affine transformation followed by a nonlinear activation function. For a given input $x_k$, the network produces a prediction $\hat{y}_k = f(x_k; \theta)$, where $f(\cdot)$ denotes the nonlinear mapping induced by the network architecture and $\theta$ represents the collection of all weights and biases. The objective of MLP regression is to estimate the parameters $\theta$ by minimizing an empirical risk function of the form
\begin{equation}
\min_{\theta} \; \frac{1}{n}\sum_{k=1}^n L\!\left(y_k - f(x_k;\theta)\right),
\end{equation}
where $L(\cdot)$ is a regression loss function that measures the discrepancy between the true target and the predicted output. The choice of the loss function plays a critical role in determining the robustness and generalization performance of the MLP, particularly in the presence of noise and outliers. Commonly used loss functions include the squared error, absolute error, and their robust variants, which provide the foundation for developing improved learning models in regression settings.

\begin{figure}[t]
\centering

\begin{minipage}{0.32\textwidth}
\centering
\includegraphics[width=\textwidth]{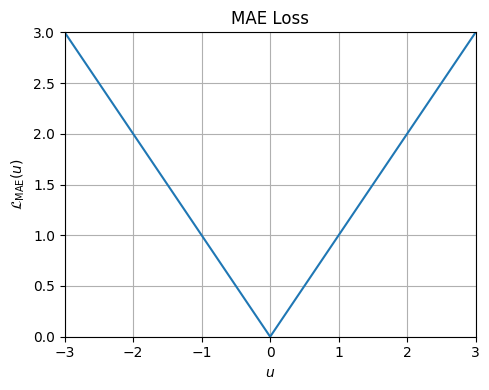}
\par\smallskip
{\small (a)}
\end{minipage}\hfill
\begin{minipage}{0.32\textwidth}
\centering
\includegraphics[width=\textwidth]{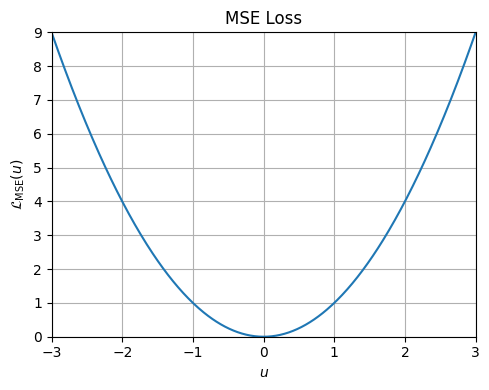}
\par\smallskip
{\small (b)}
\end{minipage}\hfill
\begin{minipage}{0.32\textwidth}
\centering
\includegraphics[width=\textwidth]{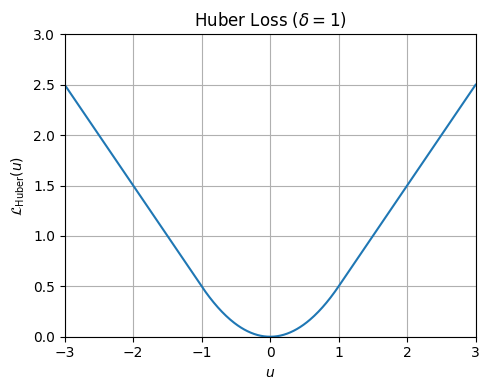}
\par\smallskip
{\small (c)}
\end{minipage}

\vspace{0.35cm}

\begin{minipage}{0.32\textwidth}
\centering
\includegraphics[width=\textwidth]{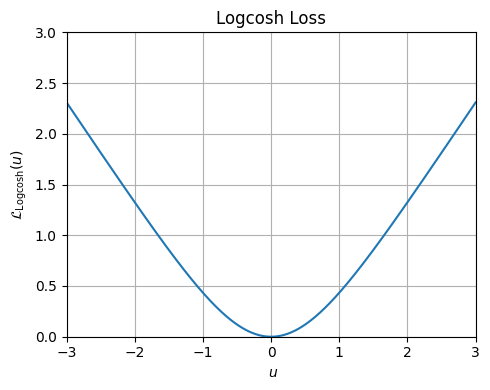}
\par\smallskip
{\small (d)}
\end{minipage}\hfill
\begin{minipage}{0.32\textwidth}
\centering
\includegraphics[width=\textwidth]{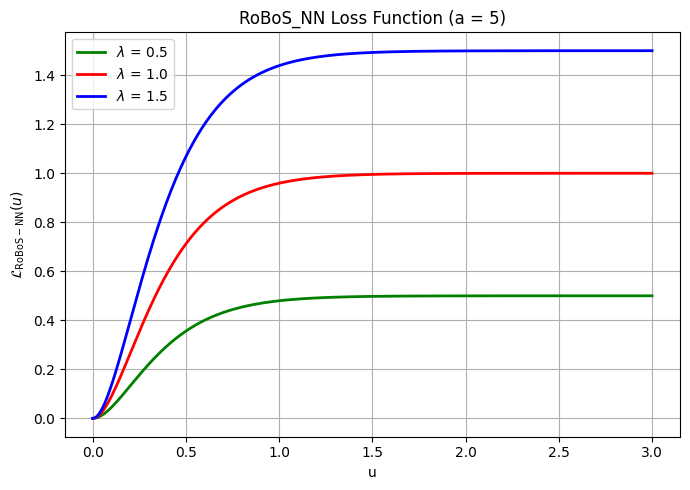}
\par\smallskip
{\small (e)}
\end{minipage}\hfill
\begin{minipage}{0.32\textwidth}
\centering
\includegraphics[width=\textwidth]{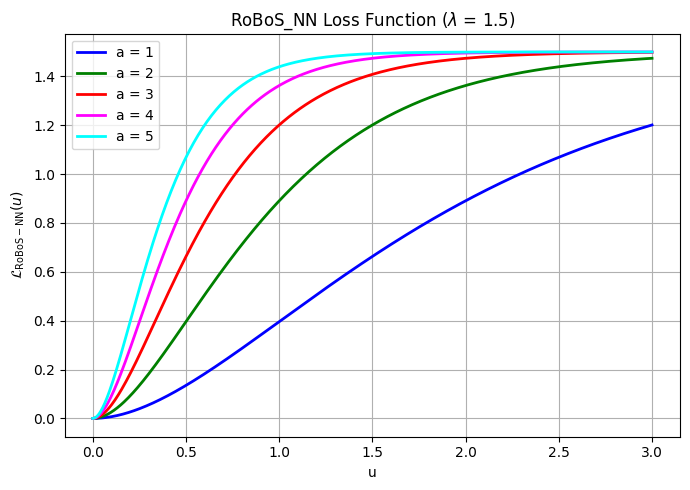}
\par\smallskip
{\small (f)}
\end{minipage}

\caption{
Comparison of loss function profiles:
(a) MAE loss,
(b) MSE loss,
(c) Huber loss,
(d) Log-cosh loss,
(e) RoBoS-NN loss with varying shape parameter $\lambda$,
(f) RoBoS-NN loss with varying robustness parameter $a$.}

\label{fig:loss_comparison}
\end{figure}


\textbf{Mean Square Loss:}
The squared loss function is one of the most widely used loss functions in forecasting tasks, which measures the squared difference between the true value and the predicted value. It is also smooth and convex in nature and is defined as follows: 
\[
\mathcal{L}_{\text{Square}} = \frac{1}{n}\sum (y_i - \hat{y}_i)^2
\]
Square loss is highly sensitive to outliers, since squaring the error amplifies larger deviations, and hence causing degradation in predictive performance 
\\

\textbf{Mean Absolute Loss:}
The mean absolute function quantifies the absolute difference between the true value and the predicted value, which is used in training neural networks for forecasting. It is defined as follows:
\[
\mathcal{L}_{\text{Absolute}}(y,\hat{y}) = \frac{1}{n}\sum |y_i - \hat{y}_i|
\]
since it does not increase rapidly with the magnitude of the prediction error, it shows better robustness than the squared loss function, especially in the presence of outliers. However, this convex loss function is not smooth near zero which makes the optimization procedure complicated.\\ 
\textbf{Huber:}
The Huber loss function can be considered as a piecewise combination of the square loss and the absolute loss functions. It introduces a threshold parameter $\delta$ that serves as a boundary for distinguishing between regular and outlier samples. The mathematical formula for Huber loss is given by:
\[
\mathcal{L}_{Huber}(y,\hat{y}) = 
\begin{cases} 
\tfrac{1}{2}(y - \hat{y})^2 & |y - \hat{y}| \leq \delta, \\ 
\delta (|y - \hat{y}| - \tfrac{1}{2}\delta) & \text{otherwise}.
\end{cases}
\]
\textbf{Log-Cosh:}
The log-cosh loss is defined as the logarithm of the hyperbolic cosine of the prediction error. For small errors, it closely approximates the quadratic loss 
$\frac{1}{2}(y-f(x))^2$, whereas for large errors, it behaves approximately as 
 $\lvert{y-f(x)}\rvert-log2$. This structure makes the log-cosh loss function similar to the Huber loss, providing robustness to outliers while maintaining smooth and adaptive gradients. A notable advantage of this loss is that it is twice differentiable everywhere, which is particularly beneficial for optimization algorithms that rely on second-order information, such as Newton-type methods. However, like the Huber loss, the gradient and Hessian of the log-cosh loss saturate for samples with large errors, which may reduce optimization efficiency or lead to practical issues, for example, fewer effective split points in tree-based models such as XGBoost. The mathematical equation for this function is as follow:
\[
\mathcal{L}_{\text{logcosh}}(y, \hat{y}) = \frac{1}{n}\sum \log\!\big(\cosh(y_i - \hat{y}_i)\big)
\]
RoBoSS (Robust, Bounded, Sparse, and Smooth) loss was recently introduced for supervised
\textit{classification} within the support vector machine framework, where it demonstrated
strong robustness and favorable theoretical guarantees \cite{Akhtar2024RoBoSS}.
The RoBoSS loss is defined as
\[
\mathcal{L}_{\text{RoBoSS}}(u)=
\begin{cases}
\lambda\left\{1-(au+1)\exp(-au)\right\}, & u>0,\\
0, & u\leq 0,
\end{cases}
\]
where $a>0$ and $\lambda>0$ denote the shape and bounding parameters, respectively.
Although effective in classification, its applicability to regression and time series forecasting with continuous-valued outputs remains unexplored.
Motivated by its boundedness and smoothness, we extend the RoBoSS loss to regression , resulting in the proposed RoBoS-NN which is defned as:
\[
L(u) =
\begin{cases}
\lambda \left\{ 1 - \left( a\sqrt{u^{2}+\epsilon} - a\sqrt{\epsilon} + 1 \right) 
\exp\!\left(-\left(a\sqrt{u^{2}+\epsilon} - a\sqrt{\epsilon}\right)\right) \right\}
\end{cases}
\]
where $u=|y-\hat{y}|$ and $y$ and $\hat{y}$ denote the actual value and the predicted value, respectively. The parameters $a$, $\lambda$, $\epsilon>0$ represent shape parameter, bound parameter and stability parameter consecutively. Thereafter, we integrate this RoBoS-NN loss function in the framework of neural network and proposed $\mathcal{L}_{\text{RoBoS}}$-NN, which is given by : 
\[
\min_{\theta}
\;\frac{\lambda}{2}\|\theta\|^2
+ \frac{1}{n}\sum_{k=1}^n
\mathcal{L}_{\text{RoBoS-NN}}\!\left(
y_k - f(x_k;\theta)
\right).
\]
The non-convex nature of the proposed $\mathcal{L}_{\text{RoBoS}}$-NN loss function poses challenges for optimizing the neural network model using conventional optimization methods. However, the smooth and differentiable structure of the RoBoS-NN loss enables the use of gradient-based optimization techniques for solving the learning problem. In this paper, we employ the Adam optimization algorithm to solve the proposed $\mathcal{L}_{\text{RoBoS}}$-NN framework, owing to its adaptive learning-rate mechanism and efficiency in handling large-scale neural network optimization problems~\cite{kingma2014adam}.

\section{Proposed work}
\begin{tabular}{lllll}
\toprule
Loss function & Robust to outliers & Bounded & Smooth & Convex \\
\midrule
Square Loss    & \textcolor{red}{\ding{55}} & \textcolor{red}{\ding{55}} & \textcolor{blue}{\ding{51}} & \textcolor{blue}{\ding{51}} \\
Absolute Loss  & \textcolor{blue}{\ding{51}} & \textcolor{red}{\ding{55}} & \textcolor{red}{\ding{55}} & \textcolor{blue}{\ding{51}} \\
Huber Loss     & \textcolor{blue}{\ding{51}} & \textcolor{red}{\ding{55}} & \textcolor{blue}{\ding{51}} & \textcolor{blue}{\ding{51}}\\
Log-Cosh Loss  & \textcolor{blue}{\ding{51}} & \textcolor{red}{\ding{55}} & \textcolor{blue}{\ding{51}} & \textcolor{blue}{\ding{51}} \\
RoBoS-NN Loss  & \textcolor{blue}{\ding{51}} & \textcolor{blue}{\ding{51}} & \textcolor{blue}{\ding{51}}  & \textcolor{red}{\ding{55}} \\
\bottomrule
\end{tabular}\\
\section{THEORETICAL EVALUATION OF THE PROPOSED ROBOS\_NN LOSS FUNCTION}
\begin{theorem}
Given, ${H_d}$ be the class of real-valued networks of depth $d$ over the domain $X$, where each parameter matrix $W_j$ has Frobenius norm atmost $M_F(j)$ and with $1$-Lipschitz, positive homogeneous activation functions. Let $f_p$ be the predictor produced by Robos-NN. Then for any, $0<\epsilon<1$, with confidence $1-\epsilon$,\\
\[{R(f_p)-R_z(f_p)\leq\frac{2aB(\sqrt{2log(2)d}+1)\prod_{j=1}^{d}M_F(j)}{e\sqrt{n}}} + \sqrt{\frac{8 ln\frac{1}\epsilon{}}{n}}\]
\end{theorem}
\begin{proof}
From theorem $8$ of \cite{bartlett2002rademacher}, for any $0<\epsilon<1$, we have\\
\[R(f^{L_{Robos\_NN}}_p)-R_z(f^{L_{Robos\_NN}}_p)\leq R_n(\tilde{\phi}   \circ {H_d}) +\sqrt{\frac{8 ln\frac{1}\epsilon{}}{n}}\]
where, $$\tilde{\phi}\circ {H_d}=\{(x,y)\rightarrow \phi(y,f(x))-\phi(y,0):f\in H_d\}$$
Now, from theorem 12 of \cite{bartlett2002rademacher}, 
\begin{equation}
R_n(\tilde{\phi}   \circ {H_d})\leq2\frac{a}{e}R_n(H_d)
\end{equation}
where $\frac{a}{e}$ is the lipschitz constant obtained from the proposed RoBos\_NN loss function.\\

Also, theorem 1 of \cite{golowich2020size} gives, 
\begin{equation}
R_n(H_d)\leq \frac{B(\sqrt{2log(2)d}+1)\prod_{j=1}^{d}M_F(j)}{\sqrt{n}}
\end{equation}
Hence, for any $0<\epsilon<1$, we have
\[{R(f_p)-R_z(f_p)\leq\frac{2aB(\sqrt{2log(2)d}+1)\prod_{j=1}^{d}M_F(j)}{e\sqrt{n}}} + \sqrt{\frac{8 ln\frac{1}\epsilon{}}{n}}\]
\end{proof}

\section{Optimization of $\mathcal{L}_{\text{RoBoS}}$-NN}

\begin{figure}[H]
\centering
\begin{tikzpicture}[
    node distance=1.4cm,
    every node/.style={font=\small},
    block/.style={
        rectangle, draw=black, rounded corners,
        minimum width=3.8cm,
        minimum height=0.9cm,
        align=center,
        fill=blue!8
    },
    bigblock/.style={
        ellipse, draw=black,
        minimum width=8.2cm,
        minimum height=2.8cm,
        align=center,
        fill=orange!10
    },
    arrow/.style={->, thick}
]

\node[block] (data) {Training Data};

\node[block, below=of data] (nn) {Neural Network\\
Feature Mapping};

\node[bigblock, below=of nn] (objective) {
\textbf{Optimization Objective}\\[4pt]
$\displaystyle
\min_{\theta}
\;\frac{\lambda}{2}\|\theta\|^2
+ \frac{1}{n}\sum_{k=1}^n
\mathcal{L}_{\text{RoBoS-NN}}\!\left(
y_k - f(x_k;\theta)
\right)$\\[6pt]
\textit{Regularization term} \hfill \textit{RoBoS-NN loss term}
};

\node[block, below=of objective] (adam) {Adam Optimizer\\
(Backpropagation)};

\node[block, below=of adam] (decision) {Forecasting Function\\
$\hat{y} = f(x;\theta)$};

\draw[arrow] (data) -- (nn);
\draw[arrow] (nn) -- (objective);
\draw[arrow] (objective) -- (adam);
\draw[arrow] (adam) -- (decision);

\node[right=0.3cm of objective] (robos) {RoBoS-NN loss term};
\draw[->, dashed] (robos) -- (objective);

\node[left=0.3cm of objective] (reg) {Regularization};
\draw[->, dashed] (reg) -- (objective);

\end{tikzpicture}
\caption{Training framework of the proposed neural-network-based forecasting model using the RoBoS loss optimized via the Adam algorithm.}
\label{fig:robos_nn_adam}
\end{figure}
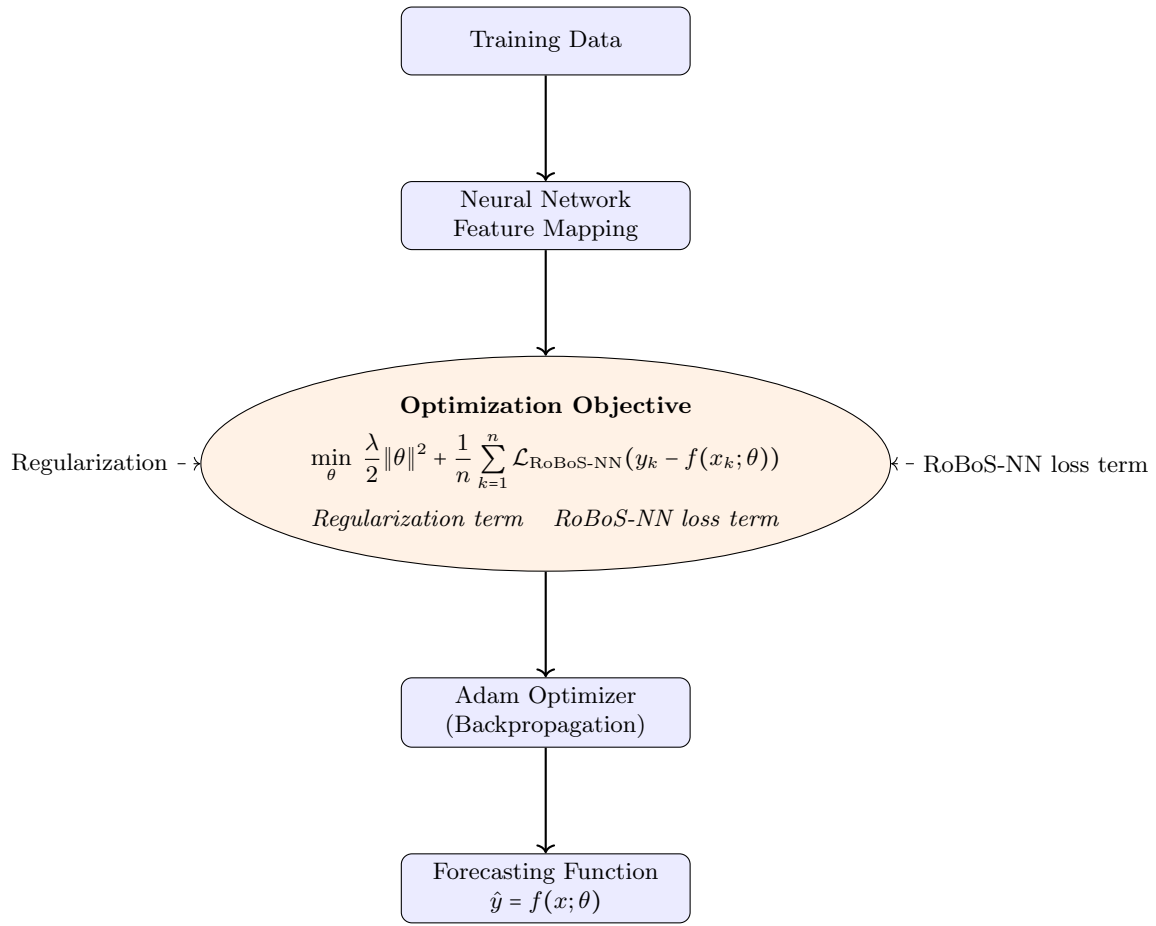

\begin{algorithm}[H]
\caption{Adam Optimization for RoBoS-NN}
\label{alg:robos_adam}
\begin{algorithmic}[1]

\REQUIRE Training dataset $\{(x_i, y_i)\}_{i=1}^{N}$,
neural network model $f_{\theta}(\cdot)$,
RoBoS-NN parameters $\lambda, a, \epsilon$,
learning rate $\eta$,
Adam parameters $\beta_1, \beta_2 \in (0,1)$,
numerical constant $\delta$,
maximum iterations $T$

\ENSURE Optimized network parameters $\theta^{*}$

\STATE Initialize network parameters $\theta_0$
\STATE Initialize first moment vector $m_0 = 0$
\STATE Initialize second moment vector $v_0 = 0$
\STATE Set time step $t = 0$

\FOR{$t = 1$ to $T$}

    \STATE Compute predictions:
    \[
    \hat{y}_i = f_{\theta}(x_i), \quad i = 1, \ldots, N
    \]

    \STATE Compute residuals:
    \[
    u_i = |y_i - \hat{y}_i|
    \]

    \STATE Compute RoBoS-NN loss:
    \[
    \mathcal{L}_i = \lambda \left[ 1 -
    \left(a\sqrt{u_i^2 + \epsilon} - a\sqrt{\epsilon} + 1\right)
    \exp\left(-\left(a\sqrt{u_i^2 + \epsilon} - a\sqrt{\epsilon}\right)\right)
    \right]
    \]

    \STATE Compute empirical risk:
    \[
    \mathcal{R}(\theta) = \frac{1}{N} \sum_{i=1}^{N} \mathcal{L}_i
    \]

    \STATE Compute gradient:
    \[
    g_t = \nabla_{\theta} \mathcal{R}(\theta)
    \]

    \STATE Update first moment estimate:
    \[
    m_t = \beta_1 m_{t-1} + (1 - \beta_1) g_t
    \]

    \STATE Update second moment estimate:
    \[
    v_t = \beta_2 v_{t-1} + (1 - \beta_2) g_t^2
    \]

    \STATE Bias correction:
    \[
    \hat{m}_t = \frac{m_t}{1 - \beta_1^t}, \quad
    \hat{v}_t = \frac{v_t}{1 - \beta_2^t}
    \]

    \STATE Parameter update:
    \[
    \theta_t = \theta_{t-1} - \eta \frac{\hat{m}_t}{\sqrt{\hat{v}_t} + \delta}
    \]

\ENDFOR

\STATE \textbf{Return} $\theta^{*} = \theta_T$

\end{algorithmic}
\end{algorithm}

\section{Computational Analysis}

The computational complexity of the proposed $\mathcal{L}_{\text{RoBoS}}$-NN is analyzed in terms of algorithmic time and space requirements. Since the network architecture and optimization strategy remain unchanged, the overall training complexity is primarily governed by the forward and backward propagation processes.

Let $N$ denote the number of training samples, $P$ the number of trainable parameters, and $T$ the number of training epochs. The forward and backward passes incur a computational cost of $\mathcal{O}(N \cdot P)$ per epoch. The RoBoS-NN loss function introduces additional scalar operations, including square-root and exponential evaluations, which incur a constant-time cost per sample. Consequently, the overall time complexity of the proposed method remains $\mathcal{O}(T \cdot N \cdot P)$, which is identical to that of standard neural networks trained with conventional loss functions such as MAE, MSE, Huber, and Log-cosh.

The Adam optimizer requires maintaining first- and second-order moment estimates for each parameter, resulting in a space complexity of $\mathcal{O}(P)$. No additional memory overhead is introduced by the RoBoS loss formulation.

Although empirical runtime measurements were not explicitly recorded, the identical training configuration across all loss functions ensures a fair computational comparison. Furthermore, the bounded and smooth nature of the RoBoS loss contributes to stable gradient updates, which can improve convergence behavior in the presence of outliers without increasing computational burden.

\section{Experimental Results}

\begin{table}[t]
\centering
\caption{MLP architecture and training hyperparameters for different datasets.}
\label{tab:training_hyperparams}
\begin{tabular}{lcccccc}
\toprule
Dataset & Seq\_Size & Dense Layers & Batch Size & Units & Learning Rate & Patience \\
\midrule
Daily\_Min\_Temperature & $30$ & $2$ & $32$ & $64$ &$0.001$  & $5$ \\
Electricity\_Load       & $96$ & $3$ & $256$ & $64$ & $0.001$ & $5$ \\
Monthly\_Sunspots       & $132$ & $3$ & $32$ & $64$ & $0.001$ & $5$ \\
Daily\_Gold\_Price      & $30$ & $2$ & $16$ & $32$ & $0.001$ & $5$ \\
\bottomrule
\end{tabular}
\end{table}

\begin{table}[t]
\centering
\caption{Hyperparameter search space and optimal values for the RoBoS-NN loss under different outlier levels}
\label{tab:roboss_outliers}
\begin{tabular}{llccc|ccc}
\toprule
Dataset & Outlier Level & 
\multicolumn{3}{c}{Search Space} & 
\multicolumn{3}{c}{Optimal Values} \\
\cmidrule(lr){3-5} \cmidrule(lr){6-8}
 &  & 
$a$ & $\epsilon$ & $\lambda$ & 
$a$ & $\epsilon$ & $\lambda$ \\
\midrule
Daily\_Min\_Temperature & 0\%  & $(1,10)$ & $(e^{-4}, 0.05)$ &$(0.1,1)$& $1.000$ & $0.028$ & $0.253$ \\
                        & 5\%  & $(1,10)$ & $(e^{-4}, 0.05)$ &$(0.1,1)$ & $1.010$ & $0.038$ & $0.168$ \\
 & 10\% & $(1,10)$ & $(e^{-4}, 0.05)$ & $(0.1,1)$ & $1.525$ & $0.037$ & $0.104$ \\
 & 20\% &$(1,10)$ & $(e^{-4}, 0.05)$ & $(0.1,1)$& $1.048$ & $0.50$ & $0.124$ \\
 & 30\% & $(1,10)$ & $(e^{-4}, 0.05)$ &$(0.1,1)$ & $1.040$ & $0.016$ & $0.104$ \\
\midrule
Electricity\_Load & 0\%  & $(1,10)$ & $(e^{-4}, 0.05)$ &$(0.1,1)$& $1.053$ & $0.041$ & $0.251$ \\
 & 5\%  & $(1,10)$ & $(e^{-4}, 0.05)$ & $(0.1,1)$ & $1.166$ & $0.028$ & $0.106$ \\
 & 10\% & $(1,10)$ & $(e^{-4}, 0.05)$ & $(0.1,1)$ & $1.010$ & $0.011$ & $0.351$ \\
 & 20\% & $(1,10)$ & $(e^{-4}, 0.05)$ & $(0.1,1)$& $1.089$ & $0.048$ & $0.100$ \\
 & 30\% & $(1,10)$ & $(e^{-4}, 0.05)$& $(0.1,1)$ & $1.025$ & $0.050$ & $0.620$ \\
\bottomrule
Monthly\_Sunspots & 0\%  & $(1,10)$ & $(e^{-4}, 0.05)$& $(0.1,1)$ & $1.327$ & $0.033$ &  $0.123$ \\
 & 5\%  & $(1,10)$ &$(e^{-4}, 0.05)$ & $(0.1,1)$& $1.557$ & $0.036$ & $0.112$ \\
 & 10\% & $(1,10)$ & $(e^{-4}, 0.05)$ & $(0.1,1)$ & $1.633$ & $0.025$ & $0.102$\\
 & 20\% & $(1,10)$ & $(e^{-4}, 0.05)$ & $(0.1,1)$ & $1.004$ & $0.012$ & $0.340$ \\
 & 30\% & $(1,10)$ & $(e^{-4}, 0.05)$ & $(0.1,1)$ & $1.008$ & $0.417$ & $0.122$ \\
\midrule
Daily\_Gold\_Price & 0\%  & $(1,10)$ & $(e^{-4}, 0.05)$ & $(0.1,1)$& $1.129$ & $0.027$ & $0.114$ \\
 & 5\%  & $(1,10)$ & $(e^{-4}, 0.05)$ & $(0.1,1)$ & $1.156$ & $0.041$ & $0.273$ \\
 & 10\% & $(1,10)$ & $(e^{-4}, 0.05)$ & $(0.1,1)$ & $2.202$ & $0.012$ & $0.105$ \\
 & 20\% & $(1,10)$ & $(e^{-4}, 0.05)$ & $(0.1,1)$ & $1.120$ & $0.040$ & $0.203$ \\
& 30\% & $(1,10)$ & $(e^{-4}, 0.05)$ & $(0.1,1)$ & $1.133$ & $0.020$ & $0.270$ \\
\bottomrule
\end{tabular}
\end{table}

\begin{table*}[t]
\centering
\caption{The prediction performance of the proposed $\mathcal{L}_{\text{RoBoS}}$-NN and competing models under different outlier levels using multiple error metrics.}
\resizebox{\textwidth}{!}{
\begin{tabular}{lccccccc}
\hline
\textbf{Dataset} & \textbf{Outliers} & \textbf{Metric} &
$\mathcal{L}_{\text{MAE}}$-NN &
$\mathcal{L}_{\text{MSE}}$-NN &
$\mathcal{L}_{\text{Huber}}$-NN &
$\mathcal{L}_{\text{Logcosh}}$-NN &
$\mathcal{L}_{\text{RoBoS}}$-NN \\
\hline

\multirow{15}{*}{Daily\_Min\_Temperature}
& \multirow{3}{*}{0\%} & MAE  & 2.286 & 2.269 & 2.282 & 2.297 & 2.330 \\
&                     & RMSE & 2.878 & 2.855 & 2.861 & 2.892 & 2.923 \\
&                     & MASE & 1.170 & 1.162 & 1.168 & 1.176 & 1.193 \\

& \multirow{3}{*}{5\%} & MAE  & 2.630 & 2.798 & 2.544 & 2.809 & 2.456 \\
&                     & RMSE & 3.277 & 3.467 & 3.193 & 3.466 & 3.095 \\
&                     & MASE & 1.342 & 1.432 & 1.302 & 1.438 & 1.257 \\

& \multirow{3}{*}{10\%} & MAE  & 3.815 & 3.147 & 3.128 & 2.993 & 2.532 \\
&                      & RMSE & 4.505 & 3.813 & 3.806 & 3.663 & 3.171 \\
&                      & MASE & 1.953 & 1.611 & 1.601 & 1.540 & 1.296 \\

& \multirow{3}{*}{20\%} & MAE  & 4.173 & 3.705 & 3.055 & 3.267 & 2.879 \\
&                      & RMSE & 4.844 & 4.396 & 3.715 & 3.939 & 3.528 \\
&                      & MASE & 2.136 & 1.897 & 1.564 & 1.673 & 1.474 \\

& \multirow{3}{*}{30\%} & MAE  & 4.096 & 3.317 & 3.317 & 2.857 & 2.629 \\
&                      & RMSE & 4.764 & 3.984 & 3.999 & 3.503 & 3.267 \\
&                      & MASE & 2.096 & 1.697 & 1.698 & 1.462 & 1.346 \\
\hline
\textbf{Total Avg.} &  & MAE  & 3.400 & 3.047 & 2.865 & 2.845 & 2.565 \\
                    &  & RMSE & 4.054 & 3.703 & 3.515 & 3.493 & 3.197 \\
                    &  & MASE & 1.739 & 1.560 & 1.467 & 1.458 & 1.313 \\
\hline

\multirow{15}{*}{Electricity\_Load}
& \multirow{3}{*}{0\%} & MAE  & 57.770 & 58.374 & 57.546 & 59.018 & 58.861 \\
&                     & RMSE & 68.038 & 68.809 & 67.787 & 69.439 & 69.374 \\
&                     & MASE & 3.082 & 3.114 & 3.070 & 3.148 & 3.141 \\

& \multirow{3}{*}{5\%} & MAE  & 145.357 & 131.154 & 136.228 & 133.343 & 113.306 \\
&                     & RMSE & 169.563 & 150.247 & 156.047 & 152.063 & 132.031 \\
&                     & MASE & 7.755 & 6.998 & 7.268 & 7.114 & 6.045 \\

& \multirow{3}{*}{10\%} & MAE  & 158.895 & 136.908 & 142.268 & 142.032 & 110.121 \\
&                      & RMSE & 182.459 & 157.100 & 123.620 & 163.081 & 128.636 \\
&                      & MASE & 8.478 & 7.305 & 6.596 & 7.578 & 5.876 \\

& \multirow{3}{*}{20\%} & MAE  & 166.700 & 122.210 & 121.945 & 126.128 & 105.783 \\
&                      & RMSE & 190.441 & 141.176 & 141.175 & 145.321 & 122.954 \\
&                      & MASE & 8.894 & 6.520 & 6.506 & 6.710 & 5.644 \\

& \multirow{3}{*}{30\%} & MAE  & 164.728 & 109.469 & 119.714 & 104.386 & 92.346 \\
&                      & RMSE & 188.465 & 126.197 & 137.636 & 122.759 & 108.278 \\
&                      & MASE & 8.789 & 5.841 & 6.387 & 5.569 & 4.927 \\
\hline
\textbf{Total Avg.} &  & MAE  & 138.690 & 111.623 & 115.540 & 112.981 & 96.083 \\
                    &  & RMSE & 159.793 & 128.706 & 125.253 & 130.533 & 112.255 \\
                    &  & MASE & 7.400 & 5.956 & 5.965 & 6.024 & 5.127 \\
\hline
\end{tabular}
}
\end{table*}

\begin{table*}[t]
\centering
\caption{The prediction performance of the proposed $\mathcal{L}_{\text{RoBoS}}$-NN and competing models under different outlier levels using multiple error metrics.}
\resizebox{\textwidth}{!}{
\begin{tabular}{lccccccc}
\hline
\textbf{Dataset} & \textbf{Outliers} & \textbf{Metric} &
$\mathcal{L}_{\text{MAE}}$-NN &
$\mathcal{L}_{\text{MSE}}$-NN &
$\mathcal{L}_{\text{Huber}}$-NN &
$\mathcal{L}_{\text{Logcosh}}$-NN &
$\mathcal{L}_{\text{RoBoS}}$-NN \\
\hline

\multirow{15}{*}{Monthly\_Sunspots}
& \multirow{3}{*}{0\%} & MAE  & 20.854 & 20.787 & 21.548 & 20.149 & 22.698 \\
&                     & RMSE & 28.417 & 27.995 & 28.009 & 27.029 & 29.952 \\
&                     & MASE & 1.408 & 1.404 & 1.455 & 1.361 & 1.533 \\

& \multirow{3}{*}{5\%} & MAE  & 54.447 & 31.118 & 31.213 & 44.892 & 35.528 \\
&                     & RMSE & 67.878 & 42.740 & 40.602 & 54.366 & 46.762 \\
&                     & MASE & 3.677 & 2.101 & 2.108 & 3.031 & 2.399 \\

& \multirow{3}{*}{10\%} & MAE  & 50.628 & 62.466 & 44.604 & 47.614 & 42.209 \\
&                      & RMSE & 63.477 & 69.638 & 58.497 & 62.411 & 52.946 \\
&                      & MASE & 3.419 & 4.218 & 3.012 & 3.215 & 2.850 \\

& \multirow{3}{*}{20\%} & MAE  & 62.163 & 64.180 & 39.157 & 43.119 & 32.227 \\
&                      & RMSE & 77.076 & 72.096 & 48.399 & 52.909 & 43.266 \\
&                      & MASE & 4.198 & 4.334 & 2.644 & 2.912 & 2.176 \\

& \multirow{3}{*}{30\%} & MAE  & 66.191 & 39.560 & 34.671 & 38.491 & 32.742 \\
&                      & RMSE & 81.144 & 52.241 & 45.830 & 52.309 & 43.167 \\
&                      & MASE & 4.470 & 2.671 & 2.341 & 2.599 & 2.211 \\
\hline
\textbf{Total Avg.} &  & MAE  & 50.857 & 43.622 & 34.239 & 38.853 & 33.081 \\
                    &  & RMSE & 63.598 & 52.942 & 44.267 & 49.805 & 43.219 \\
                    &  & MASE & 3.434 & 2.946 & 2.312 & 2.624 & 2.234 \\
\hline

\multirow{15}{*}{Daily\_Gold\_Price}
& \multirow{3}{*}{0\%} & MAE  & 22.263 & 22.687 & 23.358 & 21.578 & 19.175 \\
&                     & RMSE & 23.248 & 23.407 & 24.033 & 22.447 & 20.447 \\
&                     & MASE & 8.131 & 8.286 & 8.531 & 7.881 & 7.003 \\

& \multirow{3}{*}{5\%} & MAE  & 16.773 & 15.097 & 12.906 & 13.178 & 10.184 \\
&                     & RMSE & 17.579 & 16.195 & 14.116 & 14.416 & 11.538 \\
&                     & MASE & 6.126 & 5.514 & 4.714 & 4.813 & 3.720 \\

& \multirow{3}{*}{10\%} & MAE  & 21.150 & 25.250 & 25.649 & 21.632 & 20.752 \\
&                      & RMSE & 21.961 & 25.923 & 26.360 & 22.371 & 20.167 \\
&                      & MASE & 7.725 & 9.223 & 9.368 & 7.901 & 7.366 \\

& \multirow{3}{*}{20\%} & MAE  & 11.321 & 11.425 & 14.547 & 17.444 & 5.654 \\
&                      & RMSE & 12.560 & 12.656 & 15.597 & 18.545 & 6.641 \\
&                      & MASE & 4.135 & 4.173 & 5.313 & 6.371 & 2.065 \\

& \multirow{3}{*}{30\%} & MAE  & 17.138 & 8.103 & 4.813 & 12.257 & 5.092 \\
&                      & RMSE & 17.883 & 9.557 & 5.715 & 13.117 & 5.829 \\
&                      & MASE & 6.260 & 2.960 & 1.758 & 4.477 & 1.860 \\
\hline
\textbf{Total Avg.} &  & MAE  & 17.729 & 16.512 & 16.255 & 17.218 & 12.171 \\
                    &  & RMSE & 18.646 & 17.548 & 17.164 & 18.179 & 12.924 \\
                    &  & MASE & 6.475 & 6.031 & 5.937 & 6.289 & 4.403 \\
\hline
\end{tabular}
}
\end{table*}

\vspace{2mm}
\newpage

\section{Numerical experiments and Result Analysis}

\subsection{Evaluation on clean and outlier-contaminated time series datasets}
Here, we report experimental outcomes on 4 real-world dataset collected from Kaggle and UCI Machine Learning repositories under both clean and contaminated settings. For a comprehensive analysis of the robustness of the proposed method, varying levels of outliers ($5\%$, $10\%$, $20\%$, $30\%$) have been injected into the datasets. Thereafter, each dataset has been divided in an $80:20$ train-test split for MLP model training and forecasting. On clean datasets, RoBoS-NN forecasting errors match benchmark losses closely. Although standard losses occasionally edge out slightly, the differences stay minimal, indicating RoBoS-NN preserves accuracy in pristine conditions while delivering superior robustness against contamination. The robustness advantage of RoBoS-NN becomes most apparent when comparing total average errors across all noise levels, which illustrate overall performance under mixed clean and noisy conditions. $\mathcal{L}_{\text{RoBoS}}$-NN exhibits notable superiority on Daily\_Min\_Temperature dataset by reducing the average MAE by $24.56\%$, $15.82\%$, $10.47\%$, $9.84\%$ and by lowering the average RMSE by $21.13\%$, $13.65\%$, $9.05\%$, $8.46\%$ while the average MASE is decreased by $24.49\%$, $15.83\%$, $10.54\%$, $9.95\%$ when compared to $\mathcal{L}_{\text{MAE}}$-NN,
$\mathcal{L}_{\text{MSE}}$-NN, $\mathcal{L}_{\text{Huber}}$-NN,
$\mathcal{L}_{\text{Logcosh}}$-NN consecutively. Additonally, $\mathcal{L}{\text{RoBoS}}$-NN demonstrates remarkable performance on the Electricity\_Load dataset by minimizing the average MAE by approximately $30.7\%$, $13.9\%$, $16.8\%$, and $15.0\%$, and the average RMSE by $29.7\%$, $12.8\%$, $10.4\%$, and $14.0\%$, compared with $\mathcal{L}{\text{MAE}}$-NN, $\mathcal{L}{\text{MSE}}$-NN, $\mathcal{L}{\text{Huber}}$-NN, and $\mathcal{L}_{\text{Logcosh}}$-NN, respectively. Furthermore, the average MASE is lessened by $30.7\%$, $13.9\%$, $14.0\%$, and $14.9\%$ relative to the same benchmark losses. Moreover, $\mathcal{L}_{\text{RoBoS}}$-NN also shows greater robustness on the {Monthly\_Sunspots} dataset by consistently achieving the lowest total average errors across all evaluation metrics. Compared with $\mathcal{L}_{\text{MAE}}$-NN, $\mathcal{L}_{\text{MSE}}$-NN, $\mathcal{L}_{\text{Huber}}$-NN, and $\mathcal{L}_{\text{Logcosh}}$-NN, it reduces the average MAE by $35.00\%$, $24.17\%$, $3.38\%$, and $14.85\%$, and the average RMSE by $32.04\%$, $18.38\%$, $2.37\%$, and $13.22\%$, respectively. Also, the average MASE is decreased by $34.95\%$, $24.18\%$, $3.38\%$, and $14.88\%$, demonstrating improved relative forecasting accuracy with higher level of outliers. Besides, on the Daily\_Gold\_Price dataset,   $\mathcal{L}_{\text{RoBoS}}$-NN consistently outperforms the benchmark loss functions by achieving the lowest total average MAE, RMSE, and MASE values. Compared with $\mathcal{L}_{\text{MAE}}$-NN, $\mathcal{L}_{\text{MSE}}$-NN, $\mathcal{L}_{\text{Huber}}$-NN, and $\mathcal{L}_{\text{Logcosh}}$-NN, it reduces the average MAE by 
$31.36\%$, $26.29\%$, $25.10\%$, and $29.32\%$, and the average RMSE by $30.72\%$, $26.34\%$, $24.72\%$, and $28.93\%$, respectively. In addition, the average MASE is decreased by $31.99\%$, $27.00\%$, $25.85\%$, and $29.99\%$, highlighting the effectiveness of the proposed loss in preserving forecasting accuracy under substantial outlier contamination.

\newpage

\section{Conclusion and Future Work}
In conclusion, this paper proposed an innovative and novel RoBoS-NN loss function to mitigate the challenges in supervised learning paradigm. Since, RoBoS-NN loss is robust, bounded and smooth, it performs as a convenient tool in several machine learning and deep learning tasks. The rigorous theoretical analysis of RoBoS-NN including the generalization error bound establish this function as a credible choice for developing robust models in supervised machine learning. The hyperparameters ($a$, $\epsilon$, $\lambda$) of the loss function were optimized using the Tree-structured Parzen Estimator(TPE) implemented in the HyperOpt library. Further, the loss function has been employed in the framework of multilayer perceptron (MLP) for forecasting multiple time series datasets with and without introduced outliers. The expermental results specifically on contaminated datasets depicted the superior performance of $\mathcal{L}_{\text{RoBoS}}$-NN model\\
 The loss function lies at the core of both machine learning and deep learning models by governing how the model learns from data, manages outliers and noise, and generalizes to unseen samples. Hence, RoBoS-NN’s favorable theoretical properties make it an appealing candidate for deeper integration with modern architectures. As a future work, RoBoS-NN loss could be utilized in more advanced neural networks such as convolutional neural networks and recurrent neural networks which may lead to new algorithms with improved robustness and performance across a wide range of applications.

\section{Acknowledgments}
This work has been financially supported by the Department of Science and Technology (DST), Government of India, through the INSPIRE fellowship with no. DST/INSPIRE Fellowship/2022/IF220134. The Centre for Computational Modeling and Simulation, National Institute of Technology Calicut, has provided assistance for all computational work.

\bibliographystyle{plain}
 \bibliography{reference}

@article{bartlett2002rademacher,
  title={Rademacher and gaussian complexities: Risk bounds and structural results},
  author={Bartlett, Peter L and Mendelson, Shahar},
  journal={Journal of machine learning research},
  volume={3},
  number={Nov},
  pages={463--482},
  year={2002}
}

@article{golowich2020size,
  title={Size-independent sample complexity of neural networks},
  author={Golowich, Noah and Rakhlin, Alexander and Shamir, Ohad},
  journal={Information and Inference: A Journal of the IMA},
  volume={9},
  number={2},
  pages={473--504},
  year={2020},
  publisher={Oxford University Press}
}

@article{akhtar2024roboss,
  title={RoBoSS: A robust, bounded, sparse, and smooth loss function for supervised learning},
  author={Akhtar, Mushir and Tanveer, M and Arshad, Mohd},
  journal={IEEE Transactions on Pattern Analysis and Machine Intelligence},
  year={2024},
  publisher={IEEE}
}

@article{kingma2014adam,
  title={Adam: A method for stochastic optimization},
  author={Kingma, Diederik P},
  journal={arXiv preprint arXiv:1412.6980},
  year={2014}
}

\end{document}